\documentclass[11pt]{article}
\usepackage{longtable}
\usepackage{threeparttable}
\usepackage{fullpage} 
\usepackage{microtype}
\usepackage{graphicx}
\usepackage{subfigure}
\usepackage{booktabs} 
\usepackage{color}
\usepackage{lmodern}

\usepackage{natbib}
\usepackage{hyperref}

\usepackage{amsmath}
\usepackage{amssymb}  
\usepackage{mathtools}
\usepackage{amsthm}
 \usepackage[capitalize]{cleveref}
\usepackage{bm}
\usepackage{listings}

\newtheorem{definition}{Definition}
\newtheorem{theorem}{Theorem}
\newtheorem{corollary}{Corollary}
\newtheorem{lemma}{Lemma}
\newtheorem{observation}{Observation}
\crefname{observation}{Observation}{Observations}
\Crefname{equation}{Eq.}{Eqs.}
\Crefname{figure}{Fig.}{Figs.}
\Crefname{table}{Table}{Tables}

\def\calA{{\mathcal{A}}}
\def\calB{{\mathcal{B}}}
\def\calC{{\mathcal{C}}}
\def\calE{{\mathcal{E}}}
\def\calF{{\mathcal{F}}}
\def\calG{{\mathcal{G}}}
\def\calH{{\mathcal{H}}}
\def\calL{{\mathcal{L}}}
\def\calM{{\mathcal{M}}}
\def\calN{{\mathcal{N}}}
\def\calR{{\mathcal{R}}}
\def\calS{{\mathcal{S}}}
\def\calU{{\mathcal{U}}}
\def\calV{{\mathcal{V}}}

\def\calX{{\mathcal{X}}}

\def\GT{{\mathrm{GT}}}

\DeclareMathOperator*{\argmax}{arg\,max}

\DeclareMathOperator*{\E}{\mathbb{E}}

\def\dichotomous{{binary}}
\def\Dichotomous{{Binary}}

\def\sing{\operatorname{sr}} 
\newcommand{\gerr}{\operatorname{err}}    
\newcommand{\cerr}{\operatorname{err}_{\mathrm{iiv}}}   

\def\cerrt{\cerr}
\def\deltat{\delta}
\def\hatft{\hat{f}}

\def\MM{\mathrm{MM}}
\def\I{{\mathrm{I}}}
\def\IDK{{\mathrm{IDK}}}

\def\1{\mathbf{1}}

\def\eps{\varepsilon}

\def\opt{\mathrm{opt}}

\newcommand{\prompt}[1]{\verb+#1+}

\usepackage{graphicx} 
\usepackage{xcolor}
\usepackage{tablefootnote}   
\usepackage{enumitem}
\usepackage[singlelinecheck=false]{caption}   

\title{Why Language Models Hallucinate}
\author{Adam Tauman Kalai\thanks{Email: {\tt adam@kal.ai}}\\OpenAI \and Ofir Nachum\\OpenAI 
\and Santosh S. Vempala\thanks{Supported in part by NSF award CCF-2106444 and a Simons Investigator award. Email: {\tt vempala@gatech.edu}}\\Georgia Tech
\and Edwin Zhang \\OpenAI
}
\begin{document}
\date{September 4, 2025}
\maketitle
\begin{abstract}
Like students facing hard exam questions, large language models sometimes guess when uncertain, producing plausible yet incorrect statements instead of admitting uncertainty. Such ``hallucinations'' persist even in state-of-the-art systems and undermine trust. 
We argue that language models hallucinate because the training and evaluation procedures reward guessing over acknowledging uncertainty, and we analyze the statistical causes of hallucinations in the modern training pipeline. Hallucinations need not be mysterious---they originate simply as errors in binary classification. If incorrect statements cannot be distinguished from facts, then hallucinations in pretrained language models will arise through natural statistical pressures. 
We then argue that hallucinations persist due to the way most evaluations are graded---language models are optimized to be good test-takers, and guessing when uncertain improves test performance. This ``epidemic'' of penalizing uncertain responses can only be addressed through a socio-technical mitigation: modifying the scoring of existing benchmarks that are misaligned but dominate leaderboards, rather than introducing additional hallucination evaluations. This change may steer the field toward more trustworthy AI systems.
\end{abstract}

\section{Introduction}
Language models are known to produce overconfident, plausible falsehoods, which diminish their utility and trustworthiness. This error mode is known as ``hallucination,''  though it differs fundamentally from the human perceptual experience. Despite significant progress, hallucinations continue to plague the field, and are still present in the latest models  \citep{openai2025gpt5}. 
Consider the prompt:
$$\text{What is Adam Tauman Kalai’s birthday? If you know, just respond with DD-MM.}$$ 
On three separate attempts, a state-of-the-art open-source language model\footnote{The language model was DeepSeek-V3 (600 B parameters), accessed via the DeepSeek app on 11 May 2025.} output three incorrect dates: ``03-07'', ``15-06'', and ``01-01'', 
even though a response was requested only if known. The correct date is in Autumn. \Cref{tab:dissertation} provides an example of more elaborate hallucinations. 

Hallucinations are an important special case of \textit{errors} produced by language models, which we analyze more generally using computational learning theory \citep[e.g.,][]{kearns1994introduction}. We consider general sets of \textit{errors} $\calE$, an arbitrary subset of plausible strings $\calX = \calE\cup\calV$, with the other plausible strings $\calV$ being called \textit{valid}. We then analyze the statistical nature of these errors, and apply the results for the type of errors of interest: plausible falsehoods called hallucinations. Our formalism also includes the notion of a \textit{prompt} to which a language model must respond.

The distribution of language is initially learned from a corpus of training examples, which inevitably contains errors and half-truths. However, we show that even if the training data were error-free, the objectives optimized during language model training would lead to errors being generated. With realistic training data containing shades of error, one may expect \textit{even higher error rates}. Thus, our lower bounds on errors apply to more realistic settings, as in traditional computational learning theory \citep{kearns1994introduction}.

Our error analysis is general yet has specific implications for hallucination. It applies broadly, including to reasoning and search-and-retrieval language models, and the analysis does not rely on properties of next-word prediction or Transformer-based neural networks. It only considers the two stages of the modern training paradigm: pretraining and post-training, described below. For hallucinations, taxonomies \citep{Maynez2020,ji2023survey} often further distinguish \textit{intrinsic} hallucinations that contradict the user's prompt, such as:
$$\text{How many Ds are in DEEPSEEK? If you know, just say the number with no commentary.}$$
DeepSeek-V3 returned ``2'' or ``3'' in ten independent trials; Meta AI and Claude 3.7 Sonnet\footnote{\url{https://meta.ai} and \url{https://claude.ai}, accessed May 9, 2025.} performed similarly, including answers as large as ``6'' and ``7''. 
Our theory also sheds light on \textit{extrinsic} hallucinations, which contradict the training data or external reality.



\begin{table}[t]
\centering
\begin{tabular}{@{}p{0.11\textwidth}@{\hspace{8pt}}p{0.87\textwidth}@{}}
\hline
ChatGPT: (GPT-4o) &
Adam Tauman Kalai's Ph.D.\ dissertation (completed in 2002 at CMU) is entitled:
``Boosting, Online Algorithms, and Other Topics in Machine Learning.''  \\[6pt]
DeepSeek: &
``Algebraic Methods in Interactive Machine Learning''\ldots at Harvard University in 2005.\\[6pt]
Llama: &
``Efficient Algorithms for Learning and Playing Games''\ldots in 2007 at MIT.\\\hline
\end{tabular}
\caption{Excerpts from responses to “What was the title of Adam Kalai’s dissertation?” from three popular language models.\tablefootnote{The models were freely accessed 2025-05-09 via \url{chatgpt.com},  the DeepSeek app \citep[R1,][]{DeepSeekR1_2025}, and \url{huggingface.co} (Llama-4-Scout-17B-16E-Instruct), respectively. None of the models searched the Web.} None generated the correct title or year \citep{kalai_probabilistic_2001}.}
\label{tab:dissertation}
\end{table}

\subsection{Errors caused by pretraining}\label{sec:intro_pretrain}
During pretraining, a \textit{base model} learns the distribution of language in a large text corpus.
We show that, even with error-free training data, the statistical objective minimized during pretraining would lead to a language model that generates errors. 
Proving this is nontrivial because some models make no errors, such as one that always outputs ``I don't know'' (IDK) or one that simply memorizes and reproduces an error-free corpus. Our analysis explains what types of errors should be expected after pretraining.

To do this, we draw a connection to binary classification. Consider questions of the form ``Is this a valid language model output?'' Generating valid outputs is in some sense harder than answering these Yes/No questions, because generation implicitly requires answering ``Is this valid'' about each candidate response. Formally, we consider the Is-It-Valid (IIV) binary  classification problem which has a training set consisting of a large number of responses, each labeled either as valid ($+$) or error ($-$), as illustrated in \Cref{fig:IIV}. For this supervised learning problem, both train and test data are 50/50 mixtures of valid examples labeled as $+$ (i.e., the pretraining data since we assume it is valid) and uniformly random errors from $\calE$ labeled as $-$. We then show how any language model can be used as an IIV classifier. This in turn allows us to establish a mathematical relationship between generative errors (such as hallucinations) and IIV misclassification rate:
$$\text{(generative error rate)} \gtrsim 2 \cdot \text{(IIV misclassification rate)}.$$

\begin{figure}
    \centering
    \includegraphics[width=\linewidth]{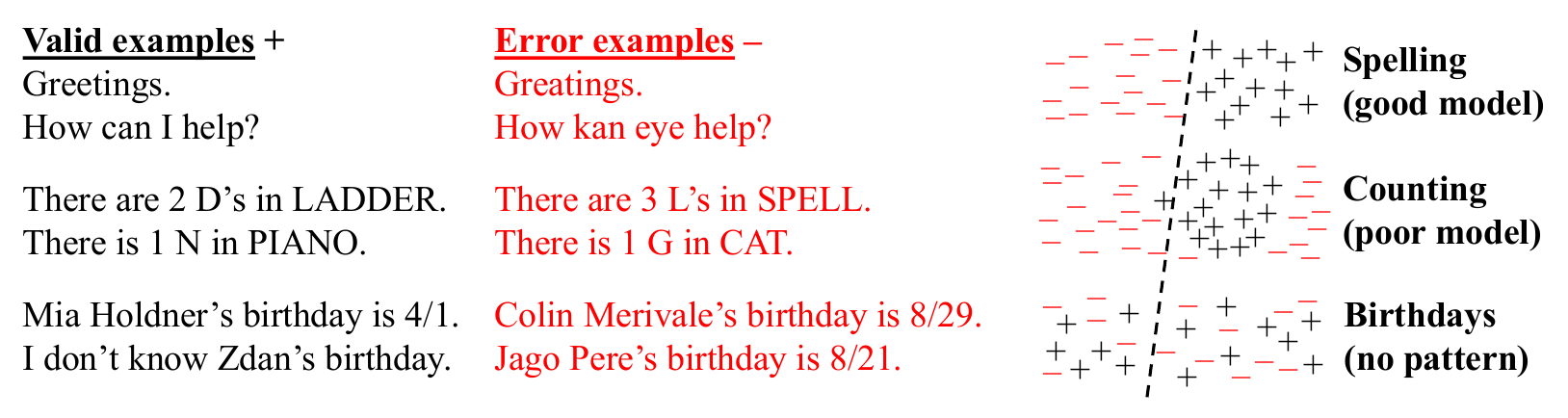}
    \caption{Is-It-Valid requires learning to identify valid generations using labeled $\pm$ examples (left). Classifiers (dashed lines) may be accurate on certain concepts like spelling (top) but errors often arise due to poor models (middle) or arbitrary facts when there is no pattern in the data (bottom).}
    \label{fig:IIV}
\end{figure}


Language models avoid many types of errors such as spelling mistakes, and not all errors are hallucinations. 
The reduction from IIV misclassification to generation illuminates the statistical nature of generative errors. The analysis shows how pretraining directly contributes to errors. Furthermore, it shows that the same  
\textit{statistical factors} contributing to errors in binary classification 
also cause language model errors. Decades of research has shed light on the multifaceted nature of misclassification errors \citep{Domingos2012UsefulThings}.
  \Cref{fig:IIV} (right) illustrates these factors visually: top, separable data classified accurately; middle, a poor model of a linear separator for a circular region; and bottom, no succinct pattern.
\Cref{sec:taxonomy} analyzes several factors, including the following stylized setting with epistemic uncertainty, when there is no pattern in the data. 

This reduction ties together earlier work which covered different types of facts. For example, \citet{kv23} considered a special case of arbitrary facts where there is no learnable pattern in the data, like the earlier birthday hallucination example. We show how the IIV reduction covers this case and recovers their bound that the hallucination rate, after pretraining, should be at least the fraction of training facts that appear once.  For instance, if 20\% of birthday facts appear exactly once in the pretraining data, then one expects base models to hallucinate on at least 20\% of birthday facts. In fact, our analysis strengthens their result to include prompts and IDK responses, both essential components of hallucination.




\subsection{Why hallucinations survive post-training}

The second stage, \textit{post-training}, refines the base model, often with a goal of reducing hallucinations. While the analysis of pretraining covered errors more generally, our analysis of post-training focuses on why overconfident hallucinations are generated rather than omitting information or expressing uncertainty such as IDK. 
We offer a socio-technical explanation for the persistence of hallucinations after post-training and discuss how the field can suppress them. 

As an analogy, consider the following context where humans also occasionally fabricate plausible-sounding information.
When uncertain, students may guess on multiple-choice exams and even bluff on written exams, submitting plausible answers in which they have little confidence. Language models are evaluated by similar tests. In both settings, guessing when unsure maximizes the expected score under a \dichotomous{} 0-1 scheme that awards 1 point for a correct answer and none for blanks or IDKs. Bluffs are often overconfident and specific, such as ``September 30'' rather than ``Sometime in autumn'' for a question about a date. 
Many language-model benchmarks mirror standardized human exams, using \dichotomous{} metrics such as accuracy or pass-rate. Optimizing models for these benchmarks may therefore foster hallucinations. Humans learn the value of expressing uncertainty outside of school, in the school of hard knocks. On the other hand, language models are primarily evaluated using exams that penalize uncertainty. Therefore, they are always in ``test-taking'' mode. Put simply, most evaluations are not aligned.

We are not the first to realize that \dichotomous{} grading does not measure hallucination. 
However, prior work on hallucination evaluation has generally sought after the elusive ``perfect hallucination eval.'' In \Cref{sec:post}, we argue that this is insufficient. We observe that existing primary evaluations overwhelmingly penalize uncertainty, and thus the root problem is the  \textit{abundance of evaluations that are not aligned}. Suppose Model A is an aligned model that correctly signals uncertainty and never hallucinates. Let Model B be similar to Model A except that it never indicates uncertainty and always ``guesses'' when unsure. Model B will outperform A under 0-1 scoring, the basis of most current benchmarks. This creates an ``epidemic'' of penalizing uncertainty and abstention, which we argue that a small fraction of hallucination evaluations won't suffice. The numerous primary evaluations must be adjusted to stop penalizing abstentions when uncertain.

\paragraph{Contributions.} We identify the main statistical drivers of hallucinations, from their pretraining origins to their post-training persistence. A novel connection between supervised and unsupervised learning demystifies their origin, even when training data contain IDK. The persistence of hallucinations, despite extensive work on the problem, is explained by the recognition that hallucination-like guessing is rewarded by most primary evaluations. We discuss statistically rigorous modifications to existing evaluations that pave the way to effective mitigation.

\section{Related work}\label{sec:related}

To the best of our knowledge, the reduction from supervised learning (binary classification) to unsupervised learning (density estimation or self-supervised learning) presented in this work is novel. The general method of reduction between learning problems, however, is a well-established technique for demonstrating that one problem is at least as hard as another \citep[see, e.g.,][]{Beygelzimer2016Reductions}.

A number of surveys and studies have explored the underlying causes of hallucination in language models.  \citet{sun2025why} cite factors such as model overconfidence~\cite{yin2023large}, decoding randomness~\cite{lee2022factuality}, snowballing effects~\cite{zhang2023language}, 
long-tailed training samples~\cite{sun2023head},
misleading alignment training~\cite{wei2023simple}, spurious correlations~\cite{li2022pre}, exposure bias~\cite{bengio2015scheduled}, the reversal curse~\cite{berglundreversal}, and context hijacking~\cite{jeong2024hijacking}. Analogous sources of error have long been studied in broader machine learning and statistical settings \citep{RussellNorvig2020}.

The most closely related theoretical work is by \citet{kv23}, which we show is a special case of our reduction. They connected the Good-Turing missing mass estimates \citep{good_population_1953} to hallucinations, which inspired \Cref{thm:agnostic}.  However, that work does not address uncertainty expressions (e.g., IDK), connections to supervised learning, post-training modifications, and their model did not include prompts.
\citet{active18nonsense} analyze an interactive learning algorithm that queries a validity oracle (e.g., a human) to agnostically train a language model that minimizes hallucinations. Their method is statistically efficient, requiring a reasonable amount of data, but not computationally efficient. Other recent theoretical studies \citep{Kalavasis2025Limits,kleinberg2024language} formalize an inherent trade-off between \textit{consistency} (avoiding invalid outputs) and \textit{breadth} (generating diverse, linguistically rich content). These works demonstrate that for broad classes of languages, any model that generalizes beyond its training data will either hallucinate invalid outputs or suffer mode collapse, failing to produce the full range of valid responses.

Several post-training techniques—such as reinforcement learning from human feedback (RLHF) \citep{ouyang2022training}, reinforcement learning from AI feedback (RLAIF) \citep{bai2022constitutionalaiharmlessnessai}, and direct preference optimization (DPO) \citep{rafailov2023dpo}—have been shown to reduce hallucinations, including conspiracy theories and common misconceptions. \citet{gekhman-etal-2024-fine} show that simple fine-tuning on novel information can initially decrease hallucination rates, only for them to later increase. Further, it has been demonstrated that both natural language queries and internal model activations encode predictive signals about factual accuracy and model uncertainty \citep[e.g.,][]{Kadavath2022LanguageM}. As discussed in our introduction, inconsistencies in a model’s answers to semantically related queries can also be leveraged to detect or mitigate hallucinations \citep{manakul-etal-2023-selfcheckgpt,xue2025verify,agrawal_language_2023}.

Numerous other methods have proven effective in mitigating hallucinations; see, for example, the surveys by \citet{ji2023survey} and \citet{Tian2024Factuality}. In terms of evaluation, several comprehensive benchmarks and leaderboards have recently been introduced \citep[e.g.,][]{bang2025hallulens,hong2024hallucinationsleaderboardopen}. 
However, relatively little work has examined barriers to their adoption. The 2025 AI Index report \citep{aiindex2025}, for instance, notes that hallucination benchmarks “have struggled to gain traction within the AI community.”

Beyond binary expressions of certainty, more nuanced linguistic constructions have been proposed to communicate gradations of uncertainty \citep{mielke-etal-2022-reducing,lin2022teaching,damani2025beyond}. Additionally, the field of pragmatics—which investigates how meaning is shaped by context—has increasing relevance for understanding and improving how language models convey information \citep{pragmatics2025}.

\section{Pretraining Errors}\label{sec:pretrain}

Pretraining produces a base language model $\hat{p}$ that approximates the distribution text drawn from its training distribution $p$. This is the classic ``density estimation'' problem in unsupervised learning, where a \textit{density} is simply a probability distribution over data. In the case of language models, the distribution is over text or multimodal inputs if included. 

The key challenge in proving that base models err is that many language models do not err. The degenerate model which always outputs IDK also avoids errors (assuming IDK is not an error). Similarly, assuming error-free training data, the trivial base model which regurgitates text from a random training example also does not err. However, these two language models fail at density estimation, the basic goal of statistical language modeling as defined below.  Errors are also avoided by the optimal base model $\hat{p} = p$ which matches the training distribution, but this model would require prohibitively large training data. Nonetheless, we show that well-trained base models should still generate certain types of errors.

Our analysis shows that generating valid outputs (i.e., avoiding errors) is harder than classifying output validity. This reduction enables us to apply the lens of computational learning theory, where errors are expected and understood, to error mechanisms in generative models. A language model is initially defined as a probability distribution over text and later \textit{prompts} are incorporated (\Cref{sec:prompts}); both settings share the same intuition. Examples without prompts include birthday statements such as those of \Cref{fig:IIV}, while a prompted model might be queried for a specific individual's birthday.

\paragraph{Not merely autocomplete.}
Our analysis applies to general density estimation and not only ``next-word predictors'' even though many language models are trained using \textit{self-supervised learning} to predict each word based on the previous words. It is tempting to attribute hallucinations to poorly chosen prefixes (e.g., ``Adam Kalai was born on'') for which the language model cannot provide valid completions. However, from a purely statistical perspective, ignoring computation, the autocomplete view\footnote{Mathematically, any distribution $p$ induces a distribution of completions $p(w_{i}w_{i+1}\ldots \mid w_1 w_2 \ldots w_{i-1})$ for every prefix of words $w_1\ldots w_{i-1}$ in its support.} of language models is no more significant than the fact that any human speaker produces one word at a time. Our analysis suggests that errors arise from the very fact that the models are being fit to the underlying language distribution, though the specific architecture can introduce additional errors.

\subsection{The reduction without prompts}\label{sec:noprompts}

Without prompts, a base model $\hat{p}$ is a probability distribution over a set $\calX$. As discussed earlier, each \textit{example} $x \in \calX$ represents a ``plausible'' string, e.g., a document.\footnote{We assume that $\calX$ is finite for simplicity. See \Cref{sec:limit} for further discussion of errors and plausibility.} 
The examples $\calX=\calE\cup \calV$ are partitioned into errors $\calE$ and valid examples $\calV$, for nonempty disjoint sets $\calE,\calV$. The  error rate of base model $\hat{p}$ is denoted by,
\begin{equation}\label{eq:gerr}
\gerr:=\hat{p}(\calE) = \Pr_{x \sim \hat{p}}[x \in \calE].
\end{equation}
Training data are assumed to come from a noiseless \textit{training distribution} $p(\calX)$, i.e., where $p(\calE)=0$. As discussed, with noisy training data and partly correct statements, one may expect \textit{even higher error rates} than our lower bounds. 

We now formalize the IIV binary-classification problem, introduced in the introduction. IIV is specified by the target function $f:\calX \rightarrow \{-,+\}$ to be learned (membership in $\calV$) and the distribution $D$ over examples $\calX$ (a 50/50 mix of samples from $p$ and uniformly random errors):
$$ D(x) :=\begin{cases}
    p(x)/2 & \text{ if } x \in \calV,\\
    1/2|\calE| & \text{ if } x\in \calE,
\end{cases}
\text{ and }
f(x) :=\begin{cases}
    + & \text{ if } x \in \calV,\\
    - & \text{ if } x \in \calE.
\end{cases}$$

Our analysis lower bounds the error rate $\gerr= \hat{p}(\calE)$ in terms of IIV's aforementioned \textit{misclassification rate} 
$\cerrt$: 
\begin{equation}\label{eq:cerrt}  \cerrt := \Pr_{x \sim D}\left[\hatft(x) \ne f(x)\right], \text{ where }  \hatft(x) :=\begin{cases}
    + & \text{ if } \hat{p}(x) > 1/|\calE|,\\
    - & \text{ if } \hat{p}(x) \le 1/|\calE|.
\end{cases}\end{equation}
The base model is thus used as an IIV classifier, in our reduction, by thresholding the base model's probability at a certain threshold $1/|\calE|$. Note that such probabilities $\hat{p}(x)$ can generally be efficiently computed for base models (though efficient computation is not necessary for the lower-bounds to be meaningful).

\begin{corollary}\label{cor:main}
    For any training distribution $p$ such that $p(\calV)=1$ and any base model $\hat{p}$,
    $$
    \gerr \ge 2\cdot \cerrt - \frac{|\calV|}{|\calE|} - \deltat,
    $$
    for $\gerr,\cerrt$ from \Cref{eq:gerr,eq:cerrt}, and $\deltat := \left|\hat{p}(\calA)-p(\calA)\right|$ for $\calA :=  \left\{x \in \calX ~\middle|~ \hat{p}(x) > 1/|\calE|\right\}$.\end{corollary}
Since this relationship holds for \textit{any} base model $\hat{p}$, it immediately implies that all base models will err on inherently unlearnable IIV facts (such as the birthdays absent from the training data) where $\cerrt$ is necessarily large, and where  $\delta$ and $|\calV|/|\calE|$ are small (e.g., for each person there are 364 times more incorrect birthday claims in $\calE$ than correct ones in $\calV$, plus IDK). The corollary above follows immediately as a special case of \Cref{thm:main} which covers the more general case with  prompts. 
\Cref{thm:indep} later uses this general result to provide lower-bounds for an intuitive special case. \Cref{thm:agnostic,thm:mc} address small $|\calE|$, e.g., $|\calE|=1$ for True/False questions. The constant 2 in the above bound is relatively tight: for large $|\calE|$ and small $\delta$, $\cerrt$ could be near 1/2 for unlearnable concepts while $\gerr \le 1$. \Cref{cor:main} also implies that $\cerrt \lesssim 1/2$.

\paragraph{Hallucination errors.} To apply the error analysis to hallucinations, one may consider $\calE$ to be the set of plausible generations containing (one or more) plausible falsehoods. Note that a common alternate definition of hallucinations is as  \textit{generations that are not grounded in the training data} (or prompt). Fortunately, the lower-bound above also applies to this notion because we have assumed only valid training data, i.e., a generated factual error cannot be grounded in factually correct training data.

\paragraph{Calibration.}
We now argue why $|\deltat|$ is a measure of (mis)calibration that is small after pretraining. 
Note that \textit{without any knowledge of the language}, one can achieve $\deltat=0$ by simply taking the uniform distribution $\hat{p}(x)=1/|\calX|$, and thus $\deltat=0$ does not require $p=\hat{p}$. An auditor can trivially estimate $\deltat$ by comparing the fractions of responses satisfying $\hat{p}(x) > 1/|\calE|$ versus $\hat{p}(\hat{x}) > 1/|\calE|$ using sets of training samples $x \sim p$ and synthetic generations $\hat{x} \sim \hat{p}$. Inspired by \citet{dawid_well-calibrated_1982}, one may think of an analogy to a weather forecaster predicting the probability of rain each day. A minimal calibration requirement would be whether their average prediction matched the average fraction of rain. One could also require these two to match on days when the forecast was $>t$ for some threshold $t \in [0,1]$. \citet{dawid_well-calibrated_1982} introduced the more stringent requirement that for \textit{every} $t \in [0,1]$, among days on which the prediction is $t$ it rains about a $t$ fraction of the time.

Here is a particularly simple justification for why $\deltat$ is typically small for the standard pretraining cross-entropy objective,
\begin{equation}\label{eq:crossent}
\calL(\hat p) = \E_{x \sim p}[-\log \hat p(x)].
\end{equation}
Consider rescaling the probabilities of the positively-labeled examples by a factor \(s > 0\) and normalizing:
\[
\hat p_s(x) :\propto
\begin{cases}
s \cdot \hat p(x) & \text{if } \hat p(x) > 1/|\calE|, \\
\hat p(x) & \text{if } \hat p(x) \le 1/|\calE|.
\end{cases}
\]
Then, a simple calculation shows that $\deltat$ is the magnitude of the derivative of the loss with respect to the scaling factor $s$, evaluated at $s=1$:
$$\deltat = \left|~ \frac{d}{ds} \calL(\hat p_s) \Big|_{s=1} ~\right|.$$
If $\deltat \ne 0$, then rescaling by some $s \ne 1$ would reduce the loss, so the loss is not at a local minimum. For any class of language models 
powerful enough to approximate such simple rescaling, local optimization should yield small $\deltat$.  
Note that $\delta$, being defined at a single threshold $t=1/|\calE|$ is weaker than notions such as Expected Calibration Error (ECE) which integrate over thresholds $t$.

\paragraph{Hallucinations are inevitable \textit{only for base models}.} 
Many have argued that hallucinations are inevitable \citep{Jones2025Hallucinations,Leffer2024SAInevitable,Xu2024HallucinationInevitable}. However,
a non-hallucinating model could be easily created, using a question-answer database and a calculator, which answers a fixed set of questions such as ``What is the chemical symbol for gold?'' and well-formed mathematical calculations such as ``3 + 8'', and otherwise outputs IDK. Moreover, the error lower-bound of \Cref{cor:main} implies that language models which do not err must not be calibrated, i.e., $\delta$ must be large. 
As our derivations show, calibration---and, hence, errors---is a natural consequence of the standard cross-entropy objective. Indeed, empirical studies (\Cref{fig:gpt4}) show that \textit{base models} are often found to be calibrated, in contrast to post-trained models which may deviate from cross-entropy in favor of reinforcement learning. 






\begin{figure}[t]
    \centering
    \makebox[0pt]{
    \includegraphics[width=0.49\linewidth]{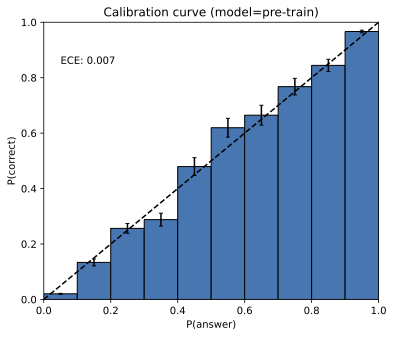}\hspace{0.02\linewidth}
    \includegraphics[width=0.49\linewidth]{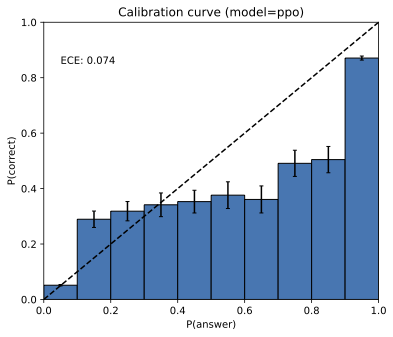}
    }
    \caption{GPT-4 calibration histograms before (left) and after (right) reinforcement learning \citep[][Figure 8, reprinted with permission]{openai_gpt-4_2023}. These plots are for multiple-choice queries where the plausible responses are simply A, B, C, or D. The pretrained model is well calibrated.}
    \label{fig:gpt4}
\end{figure}

\subsection{The reduction with prompts}\label{sec:prompts}
Henceforth, we generalize the setting of \Cref{sec:noprompts} to include prompts (contexts) $c \in \calC$ drawn from a \textit{prompt distribution} $\mu$. Each example $x=(c,r)$ now consists of a prompt $c$ and plausible response $r$. The analysis above corresponds to the special case in which $\mu$ assigns probability 1 to the empty prompt. For a given prompt $c \in \calC$, let $\calV_c:=\{r \mid (c,r) \in \calV\}$ be the valid responses and $\calE_c:=\{r \mid (c,r) \in \calE\}$ be the erroneous responses. The training distribution and base model are now conditional response distributions $p(r \mid c), \hat{p}(r \mid c)$. For notational convenience, we extend these to joint distributions on $\calX$ by $p(c,r) := \mu(c)p(r \mid c)$  and $\hat{p}(c,r) := \mu(c)\hat{p}(r \mid c)$, so that  still $\gerr:= \hat{p}(\calE) = \sum_{(c,r) \in \calE}\mu(c)\hat{p}(r \mid c)$ and $p(\calE)=0$. 

Training distribution examples therefore correspond to valid ``dialogues,'' as in the case of distillation \citep{vicuna2023,anand2023gpt4all}. Although assuming that the training data contain model dialogues drawn from the same prompt distribution is unrealistic, even higher error rates may be expected when the assumption fails. The IIV problem with prompts has the same target function $f(x):=+$ iff $x \in \mathcal{V}$, but the generalized distribution $D$ selects, with equal probability either $x \sim p$ or $x=(c,r)$ for $c \sim \mu$ and uniformly random $r \in \calE_c$. 
Finally, the classifier $\hat{f}(c, r)$ is now $+$ iff $\hat{p}(r \mid c) > 1/\min_c |\calE_c|$. 
\Cref{cor:main} is thus clearly a special case of,
\begin{theorem}\label{thm:main}
    For any training distribution $p$ such that $p(\calV)=1$ and any base model $\hat{p}$,
    $$
    \gerr \ge 2\cdot \cerrt - \frac{\max_c|\calV_c|}{\min_c|\calE_c|} - \deltat,
    $$
    where $\deltat := \left|\hat{p}(\calA)-p(\calA)\right|$ for $\calA :=  \left\{(c, r) \in \calX ~\middle|~ \hat{p}(r \mid c) > 1/\min_c |\calE_c|\right\}.$ 
\end{theorem}
Generalizing the rescaling $\hat{p}_s(r \mid c)$ (normalizing per prompt, still with single parameter $s$) again justifies a small $\deltat=\bigl|\frac{d}{ds}\calL(\hat{p}_s)|_{s=1}\bigr|$, now for $\calL(\hat{p}) := \sum_{(c, r)\in \calX} -\mu(c) \log \hat{p}(r \mid c)$.

\subsection{Error factors for base models}\label{sec:taxonomy}
Decades of research have elucidated the statistical factors contributing to misclassifications (errors in binary classification). We can leverage this prior understanding to enumerate factors behind hallucinations and other generative errors, including: statistical complexity, as in birthdays (\Cref{sec:indep}); \textit{poor models}, as in letter counting (\Cref{sec:trigram}); and additional factors like GIGO, as in conspiracy theories (\Cref{sec:additional}).

\subsubsection{Arbitrary-fact hallucinations}\label{sec:indep}

When there is no succinct pattern that explains the target function, there is epistemic uncertainty meaning that necessary knowledge is absent from the training data. The Vapnik-Chervonenkis dimension \citep{VC} $\mathrm{VC}(\calF)$ characterizes the worst-case number of examples needed to learn a family $\calF$ of functions $f: \calX \rightarrow \{-,+\}$, with high probability. Families with high $\mathrm{VC}(\calF)$ dimension may require prohibitively many samples to learn. We consider a natural special case of high VC dimension: random arbitrary facts. In particular, this section considers valid responses (other than IDK) which are random and independent across prompts. 
\begin{definition}[Arbitrary Facts]
The following are fixed: an arbitrary prompt distribution $\mu(c)$, an $\IDK$ response and, for each prompt $c$: a response set $\calR_c$ and a probability of answering $\alpha_c \in [0,1]$. Independently for each $c$, a single correct answer $a_c \in \calR_c$ is chosen uniformly at random. Finally, $p(a_c \mid c)=\alpha_c$ and $p(\IDK \mid c) = 1-\alpha_c$ for each $c \in \calC$. Thus $\calE_c = \calR_c \setminus \{a_c\}$ and $\calV_c = \{a_c, \IDK\}$.
\end{definition}
It is assumed that there is a single way to write any given fact, which can be done as in the lead birthday example where the format was specified. However, we again note that one may expect \textit{even more hallucinations} with multiple ways to state each fact. In the case of fixed-format birthdays, $|\mathcal{E}_c|=364$ and notable people whose birthdays are discussed often would have high $\mu(c)$.
Notable birthdays like Einstein's appear multiple times, whereas others may only occur once, e.g., in an obituary. Large language models seldom err on frequently referenced facts, e.g., Einstein's birthday or dissertation title.

Our lower-bound for hallucinations is based on the fraction of prompts appearing just once in the training data, ignoring  IDKs. 
\begin{definition}[Singleton rate]\label{def:sing}
A prompt $c \in \calC$ is a \emph{singleton} if it appears exactly once in the $N$ training data $\bigl\langle(c^{(i)}, r^{(i)})\bigr\rangle_{i=1}^N$ without abstention, i.e., $|\{i : c^{(i)} = c \wedge r^{(i)} \ne \IDK\}|=1$. Let $\calS\subseteq \calC$ denote the set of singletons and $$\sing=\frac{|\calS|}{N}$$ denote the fraction of training singletons.
\end{definition}
The singleton rate builds on Alan Turing's elegant ``missing-mass'' estimator \citep{good_population_1953}, which gauges how much probability is still assigned to outcomes that have not yet appeared in a sample from a distribution. Concretely, Turing's estimate of the unseen-event probability is the fraction of samples appearing exactly once. Intuitively, singletons act as a proxy for how many more novel outcomes you might encounter in further sampling, so their empirical share becomes the estimate for the entire “missing’’ portion of the distribution. We now state our bounds for Arbitrary Facts.
\begin{theorem}[Arbitrary Facts]\label{thm:indep}
In the Arbitrary Facts model, any algorithm which takes $N$ training samples and outputs $\hat{p}$ satisfies, with probability $\ge 99\%$ over $\vec{a}=\langle a_c\rangle_{c \in \calC}$ and the $N$ training examples:
$$\gerr \ge \sing   - \frac{2}{\min_c |\calE_c|}- \frac{35 + 6 \ln N}{\sqrt{N}}- \deltat.$$ Moreover, there is an efficient algorithm outputting calibrated $\hat{p}$ ($\delta=0$) that with probability $\ge 99\%$,
$$\gerr \le\sing - \frac{\sing}{\max_c |\calE_c| +1} + \frac{13}{\sqrt{N}}.$$
\end{theorem}
An earlier version of this paper presented a related theorem that omitted prompts and abstentions \citep{kv23}. The proof is in \Cref{ap:arbitrary}. Follow-up work by \citet{miao2025hallucination} provides an empirical study of hallucinations, singleton rate, and calibration.

\subsubsection{Poor models}\label{sec:trigram}

Misclassifications can also arise when the underlying model is poor because: (a) the model family cannot represent the concept well, such as linear separators approximating circular regions, or (b) the model family is sufficiently expressive but the model itself is not a good fit. Agnostic Learning \citep{kearns_toward_1994} addresses (a) by defining the minimal error rate of any classifier in a given family $\calG$ of classifiers $g:\calX \rightarrow\{-,+\}$:
$$\mathrm{opt}(\calG):= \min_{g \in \calG} \Pr_{x \sim D}[g(x) \ne f(x)] \in [0,1].$$
If $\mathrm{opt}(\calG)$ is large, then any classifier in $\calG$ will have high misclassification rate. In our case, given a language model $\hat{p}_\theta$ parameterized by $\theta\in \Theta$, consider the family of thresholded-language-model classifiers:
$$\calG := \bigl\{g_{\theta,t} ~\bigm|~ \theta \in \Theta, t\in [0,1]\bigr\}, \text{ where }g_{\theta,t}(c, r) := \begin{cases}
+ & \text{ if }\hat{p}_\theta(r \mid c) > t,\\
- & \text{ if }\hat{p}_\theta(r \mid c) \le t.
\end{cases}$$
It follows immediately from \Cref{thm:main} that
$$
    \gerr \ge 2\cdot \mathrm{opt}(\calG) - \frac{\max_c|\calV_c|}{\min_c|\calE_c|} - \deltat.
$$
When exactly one correct response exists per context  (i.e., standard multiple choice, without IDK), the calibration term can be removed and bounds can be achieved even for $C=2$ choices.
\begin{theorem}[Pure multiple-choice]\label{thm:agnostic}
Suppose $|\calV_c|=1$ for all $c\in \calC$ and let $C=\min_c |\calE_c|+1$ be the number of choices. Then, 
    $$
    \gerr \ge 2\left(1-\frac{1}{C}\right)\cdot \mathrm{opt}(\mathcal{G})  
    $$
\end{theorem}
To illustrate, consider the classic trigram language model where each word was predicted based only on the prior two words, i.e., a context window of just two words. Trigram models were dominant in the 1980s and 1990s. Trigram models, however, regularly output ungrammatical sentences. Consider the following prompts and responses:
\begin{align*}
c_1&=\text{She lost it and was completely out of}\ldots&c_2&=\text{He lost it and was completely out of}\ldots\\
r_1&=\text{her mind.}&r_2&=\text{his mind.}
\end{align*}
Here, $V_{c_1}:=E_{c_2}:=\{r_1\}$ and $V_{c_2}:=E_{c_1}:=\{r_2\}$.
\begin{corollary}\label{lem:trigram}
Let $\mu$ be uniform over $\{c_1,c_2\}$. Then any trigram model must have a generation error rate of at least 1/2.  
\end{corollary}
This follows from \Cref{thm:agnostic} because $C=2$ and $\opt(\calG)=1/2$ for trigram models. The proofs of \Cref{thm:agnostic} and \Cref{lem:trigram} are in \Cref{ap:trigram}. Although $n$-gram models can capture longer-range dependencies for larger $n$, data requirements scale exponentially in $n$.

We now revisit the letter-counting example from the introduction. To see that this is a poor model issue, note that the DeepSeek-R1 reasoning model reliably counts letters, e.g., producing a 377-chain-of-thought that includes:
\begin{quote}
Let me spell it out: D-E-E-P-S-E-E-K.\\
First letter: D — that’s one D. Second letter: E — not D. Third letter: E — not D\ldots \\
So, the number of Ds is 1.
\end{quote}
Assuming similar training data, this suggests that R1 is a better model for the task than the DeepSeek-V3 model. One representational challenge that reasoning overcomes is that modern language models represent prompts by \textit{tokens}, e.g., D/EEP/SEE/K, rather than individual characters \citep{DeepSeekR1_2025}. 

\subsection{Additional factors}\label{sec:additional}

Errors may occur due to a combination of multiple factors, including the ones discussed above and several others. Here, we highlight a few.

\begin{itemize}
    \item \textbf{Computational Hardness}. No algorithm run on a classical computer, even an AI with superhuman capabilities, can violate the laws of computational complexity theory. Indeed, AI systems have been found to err on computationally hard problems \citep{Xu2024HallucinationInevitable}. \Cref{obs:decrypt} of \Cref{ap:hard} illustrates how \Cref{thm:main} applies to intractable queries of the form \textit{``What is the decryption of $c$?''} 
    and IDK is a valid answer. 

    \item \textbf{Distribution shift}. A well-known challenge in binary classification is that training and test data distributions often diverge \citep{QuioneroCandela2009,MorenoTorres2012}. Analogously, errors in language models often stem from out-of-distribution (OOD) prompts that differ substantially from the training distribution. A question such as, ``What’s heavier, a pound of feathers or a pound of lead?'' may be unlikely in the training data and may induce erroneous answers in certain models. Similarly, distribution shift could be a factor in the letter-counting example above, though the fact that reasoning models correctly count letters suggests that the poor models may be a greater factor.

    \item \textbf{GIGO: Garbage in, Garbage out}.
Large training corpora often contain numerous factual errors, which may be replicated by base models. The statistical similarity of GIGO for both classification and pretraining is self evident, and hence we do not provide a formal treatment. However, it is important to recognize GIGO among statistical factors, as language models have been shown to replicate errors from training data \citep{lin-etal-2022-truthfulqa,levy-etal-2021-investigating,alber2025medical}.
\end{itemize}

GIGO also offers a natural segue 
to the topic of post-training, which decreases 
certain GIGO errors, such as common misconceptions and conspiracy theories \citep{ouyang2022training,openai_gpt-4_2023,Costello2024Durably}. The next section explains why some hallucinations persist---and may even be exacerbated---by current post-training pipelines.

\section{Post-training and hallucination}\label{sec:post}

Post-training should shift the model from one which is trained like an autocomplete model to one which does not output confident falsehoods (except when appropriate, e.g., when asked to produce fiction). However, we claim that further reduction of hallucinations is an uphill battle, since existing benchmarks and leaderboards reinforce certain types of hallucination. We therefore discuss how to stop this reinforcement. This is a socio-technical problem in the sense that, not only do the existing evaluations need to be modified, but these changes need to be adopted in the influential leaderboards. 

\subsection{How evaluations reinforce hallucination}
\Dichotomous{} evaluations of language models impose a  false right-wrong dichotomy, award no credit to answers that express uncertainty, omit dubious details, or request clarification. Such metrics, including accuracy and pass rate, remain the field’s prevailing norm, as argued below. Under \dichotomous{} grading, abstaining is strictly sub-optimal. IDK-type responses are maximally penalized while an overconfident ``best guess'' is optimal. The motivation combines two desirable factors: (a) the rate of accuracy among what is output by the language model, and (b) how comprehensive responses are. However, weighing (a) more than (b) is important for reducing hallucinations.

Formally, for any given question in the form of a prompt $c$, denote the set of plausible responses (valid or error) by $\calR_c:= \{r \mid (c, r) \in \calX\}$. Further, suppose there is a set of plausible abstention responses $\calA_c \subset \calR_c$ (e.g., IDK). A \textit{grader} $g_c:\calR_c \rightarrow \mathbb{R}$ is said to be \textit{\dichotomous{}} if $\{g_c(r) \mid r \in \calR_c\} = \{0,1\}$ and $g_c(r)=0$ for all $r \in \calA_c$. A \textit{problem} is defined by $(c,\calR_c,\calA_c,g_c)$ where the test-taker knows $c, \calR_c, \calA_c$. We assume that the test-taker knows that the rubric is \dichotomous{} but is not told the correct answers, where $g_c(r)=1$. The test-taker's beliefs about the correct answer can be viewed as a posterior distribution $\rho_c$ over \dichotomous{} $g_c$'s. For any such beliefs, the optimal response is not to abstain.
\begin{observation}\label{obs:misaligned}
    Let $c$ be a prompt. For any distribution $\rho_c$ over \dichotomous{} graders, the optimal response(s) are not abstentions, i.e., $$\calA_c \cap \argmax_{r \in \calR_c}\E_{g_c \sim \rho_c}[g_c(r)] = \emptyset.$$
\end{observation}

Although the proof is trivial (see \Cref{sec:postanal}), \Cref{obs:misaligned} suggests that \textit{existing evaluations may need to be modified}. \Cref{tab:prevalence} summarizes the short meta-evaluation analysis in \Cref{sec:prevalence}, finding that the vast majority of popular evaluations have \dichotomous{} grading. Therefore, additional hallucination evaluations may not suffice when the primary evaluations penalize honestly reporting confidence and uncertainty. This does not diminish existing work on hallucination evaluations but rather points out that even the ideal hallucination evaluation and ideal post-training methodology, yielding honest reports of uncertainty, may still be drowned out because of inferior performance on the vast majority of the existing evaluations.

\begin{table}[ht]
\centering
\begin{threeparttable}
\caption{Summary of evaluation benchmarks analyzed in this work and their treatment of abstentions. “\Dichotomous{} grading’’ indicates that the primary metric is a strict correct/incorrect accuracy; “IDK credit’’ denotes whether abstentions can earn any credit.
\label{tab:prevalence}}
\begin{tabular}{@{}l l c c@{}}
\toprule
\textbf{Benchmark} & \textbf{Scoring method} & \textbf{\Dichotomous{} grading} & \textbf{IDK credit} \\ 
\midrule  
GPQA              & Multiple‑choice accuracy     & Yes & None \\
MMLU‑Pro          & Multiple‑choice accuracy     & Yes & None \\
IFEval            & Programmatic instruction verification    & Yes\tnote{a}  & None \\
Omni‑MATH         & Equivalence grading$^*$    & Yes & None \\
WildBench         & LM‑graded rubric$^*$              & No  & Partial\tnote{b} \\
BBH               & Multiple-choice / exact‑match         & Yes & None \\
MATH (L5 split)   & Equivalence grading$^*$               & Yes & None \\
MuSR              & Multiple‑choice accuracy               & Yes & None \\
SWE‑bench         & Patch passes unit tests            & Yes & None \\
HLE               & Multiple-choice / equivalence grading$^*$               & Yes & None \\
\bottomrule
\end{tabular}
\begin{tablenotes}
\footnotesize
\item[$*$] Grading is performed using language models, hence incorrect \textit{bluffs} may occasionally be scored as correct. 
\item[a] IFEval aggregates several \dichotomous{} rubric sub‑scores into a composite score.
\item[b] Grading rubric (1-10 scale) suggests that IDK may score lower than ``fair'' responses with hallucination, reinforcing hallucination.
\end{tablenotes}
\end{threeparttable}
\end{table}

\subsection{Explicit confidence targets}

Human tests are similarly mostly \dichotomous{}, and it has been recognized that they also reward overconfident bluffing. Of course, exams are only a small component of human learning, e.g., fabricating birthdays will quickly result in embarrassment. Nonetheless, some standardized national exams operate or have operated using penalties for incorrect answers (or equivalently partial credit for abstaining), including Indian JEE, NEET, and GATE exams; AMC tests from the Mathematical Association of America; and US standardized SAT, AP, and GRE tests in earlier years. Importantly, the grading system is clearly stated in the instructions, and test takers are often aware of the confidence threshold beyond which it makes sense to make their best guess. 

Similarly, we propose evaluations explicitly state \textit{confidence targets} in their instructions, within the prompt (or system message). For example, one could append a statement like the following to each question:
\begin{quote}
    Answer only if you are $>t$ confident, since mistakes are penalized $t/(1-t)$ points, while correct answers receive 1 point, and an answer of ``I don't know'' receives 0 points. 
\end{quote}
There are several natural values of $t$ including $t=0.5$ (penalty 1), $t=0.75$ (penalty 2), and $t=0.9$ (penalty 9). A threshold of $t=0$ corresponds to \dichotomous{} grading and could be described by, e.g., ``Make your best guess even if you are unsure, as if you were taking an exam.'' A simple calculation shows that the expected score of offering an answer beats IDK (score 0) iff its confidence (i.e., probability of being correct) is $> t$. 

Such penalties have been well-studied within hallucination research \citep{ji2023survey}. However, we suggest two subtle variations which have statistical ramifications. First, we propose making the confidence threshold explicit in the instructions, whereas the prior work has largely omitted mentioning the confidence targets or penalties in the instructions. (A notable exception is the work of \citet{wu2025answerrefuseguessinvestigating} who introduce ``risk-informing'' prompts with explicit penalties.)  
The ideal penalty might reflect likely real-world harms, but that is impractical as it is specific to the problem, the target application, and the user group. Without transparent specification within the instructions, it would be difficult to achieve consensus among language-model creators on the correct thresholds. Similarly, students might bicker that grading is unfair given instructions that there is an unspecified penalty for errors. Instead, specifying confidence thresholds explicitly in each problem's instructions supports objective grading even if the specific thresholds chosen are somewhat arbitrary or even random. A single model may be best across all thresholds, if the threshold is explicit. However, if the threshold is not stated, then there is an inherent tradeoff, and no single model will be best in general (other than one that is always correct).

Second, we suggest incorporating confidence targets into existing mainstream evaluations, such as the popular SWE-bench \citep{jimenez2024swebench} which involves \dichotomous{} grading of software patches, while the majority of prior work has introduced implicit error penalties in bespoke hallucination evaluations. Merely adding evaluations with implicit error penalties faces the aforementioned accuracy-error tradeoff. On the other hand, incorporating confidence targets into the established evaluations, already in use, reduces the penalty for appropriate expressions of uncertainty. It may thus amplify the effectiveness of hallucination-specific evaluations.

With explicit confidence targets, there is one behavior which is simultaneously optimal for all targets---outputting IDK among examples where its correctness probability is greater than the target. Let us refer to this as \textit{behavioral calibration}--rather than requiring the model to output a probabilistic confidence \citep{lin2022teaching}, it must formulate the most useful response in which it is at least $t$ confident.
Behavioral calibration can be audited by comparing accuracy and error rates across thresholds, and circumvents the problem that there may be exponentially many ways to phrase correct responses \citep{nature_halluc_24}. Existing models may or may not exhibit behavioral calibration, but it may prove useful as an objective evaluation.

\section{Discussion and limitations}\label{sec:limit}

It is difficult for the field to agree upon how to define, evaluate and reduce hallucinations due to their multifaceted nature. A statistical framework must prioritize certain aspects and omit others, for simplicity. Several notes are in order about the extent and limitations of the framework used herein.

\vspace{-0.1in}
\paragraph{Plausibility and nonsense.} A hallucination is a plausible falsehood, and by considering only plausible strings $\calX$, our analysis ignores the possibility of generating nonsensical strings (which state-of-the-art language models rarely generate). However, the statement and proof of \Cref{thm:main} hold with the modified definitions of nonsensical examples $\calN$ with partition $\calX=\calN \cup \calE \cup \calV$, $\gerr:=\hat{p}(\calN \cup \calE)$, $D(\calN)=0$, and the assumption that $p(\calV)=1$. 

\vspace{-0.1in}
\paragraph{Open-ended generations.} For simplicity, the examples presented in this paper are oriented towards a single factual question. However, hallucinations often arise for open-ended prompts, such as ``Write a biography about\ldots.'' This can be fit into our framework by defining a response containing one or more falsehoods to be an error. However, in such a case it would be natural to consider degrees of hallucination depending on how many errors there are.


\vspace{-0.1in}
\paragraph{Search (and reasoning) are not panaceas.} A number of studies have shown how language models augmented with search or Retrieval-Augmented Generation (RAG) reduce hallucinations \citep{lewis2020rag,shuster2021rahc,nakano2021webgpt,zhang2025ragreview}. However, \Cref{obs:misaligned} holds for arbitrary language models, including those with RAG. In particular, the \dichotomous{} grading system itself still rewards guessing whenever search fails to yield a confident answer. Moreover, search may not help with miscalculations such as in the letter-counting example, or other intrinsic hallucinations.

\vspace{-0.1in}
\paragraph{Latent context.} Some errors 
cannot be judged by the prompt and response alone. For example, suppose a user asks a question about phones and the language model provides a response about cellphones, but the question was intended to be about land lines. Such ambiguities do not fit our error definition which does not depend on context external to the prompt and response. It would be interesting to extend the model to allow for ``hidden context'' that are not part of the prompt given to the language model, but which could be used for judging errors, relating to \textit{aleatoric uncertainty}.

\vspace{-0.1in}
\paragraph{A false trichotomy.} Our formalism does not distinguish between errors of different magnitudes or degrees of uncertainty. Clearly, the correct/incorrect/IDK categories are also incomplete. Although the statistical ideal might be to score each evaluation just as we would like to score the language model in the downstream application, explicit confidence targets offer a practical, objective modification to mainstream evaluations, and a false trichotomy may at least offer an IDK option unlike a false dichotomy.

\vspace{-0.1in}
\paragraph{Beyond IDK.} There are numerous ways to signal uncertainty, such as hedging, omitting details, and asking questions. Ultimately language models may adhere to confidence notions such as linguistic calibration \citep{mielke-etal-2022-reducing,damani2025beyond}. However, the pragmatic phenomena of language \citep{austin1962_how,grice1975_logic} are nuanced. For example, while there are instances where it may be useful for language models to explicitly state probabilistic confidence estimates \citep{lin2022teaching}, this can also lead to unnatural utterances, such as, ``I'm 1/365 certain that Kalai's birthday is March 7th.'' The present paper focuses on the statistical factors regarding the top-level decision of what is said.

\section{Conclusions}

This paper demystifies hallucinations in modern language models, from their origin during pretraining to their persistence through post-training.  In pretraining, we show that generative errors parallel misclassifications in supervised learning, which are not mysterious, and naturally arise due to the minimization of cross-entropy loss.

Many language model shortcomings can be captured by a single evaluation. For example, overuse of the opener ``Certainly'' can be addressed by a single \textit{``Certainly'' eval} \citep{fridman2024amodei} because starting responses with ``Certainly'' does not significantly impact other evaluations. In contrast, we argue that the majority of mainstream evaluations reward hallucinatory behavior. Simple modifications of mainstream evaluations can realign incentives, rewarding appropriate expressions of uncertainty rather than penalizing them. This can remove barriers to the suppression of hallucinations, and open the door to future work on nuanced language models, e.g., with richer pragmatic competence \citep{pragmatics2025}.

\paragraph{Acknowledgments.} We would like to thank Alex Beutel, Tom Cunningham, Yann Dubois, Parikshit Gopalan, Johannes Heidecke, Zoe Hitzig, Saachi Jain, Manas Joglekar, Sanjay Kairam, Ehud Kalai, Amin Karbasi, Alan Luo, Anay Mehrotra, Eric Mitchell, Cameron Raymond, David G.\ Robinson, Mandip Shah, Joshua Vendrow, Grigoris Velegkas, Rose Wang, Zhigang Wang, Jason Wolfe, and Jason Wei for helpful discussions.

\bibliographystyle{ACM-Reference-Format}
\bibliography{references}
\newpage
\appendix

\section{Proof of the main theorem}\label{ap:mainproof}

We now prove the main theorem.
\begin{proof}[Proof of \Cref{thm:main}]
Let $K:=\min_{c \in \calC} |\calE_c|$ and  $k:=\max_{c \in \calC} |\calV_c|$. Also, recall that 
$\deltat=|\hat{p}(\calA)-p(\calA)|$ which can equivalently be written as $\deltat=|p(\calB)-\hat{p}(\calB)|$,
where $\calA,\calB$ denote responses that are \textit{above} and \textit{below} threshold:
\begin{align}
\calA &:=\{(c, r) \in \calX \mid \hat{p}(r \mid c) > 1/K\} \label{eq:alpha}\\
\calB &:=\{(c, r) \in \calX \mid \hat{p}(r \mid c) \le 1/K\}.\label{eq:beta} 
\end{align}
Partition the hallucination and misclassification rates into above and below threshold rates:
\begin{align*}
    \gerr &= \hat{p}(\calA \setminus \calV) + \hat{p}(\calB \setminus \calV)\\
    \cerr &= D(\calA \setminus \calV) + D(\calB \cap \calV).
\end{align*}

Above the threshold, misclassifications $D(\calA \setminus \calV)$ are the sum of $D(c,r)$ only over $(c,r) \in \calA$ such that $r \in \calE_c$---each contributing $D(c,r)=\mu(c)/2|\calE_c| \le \mu(c)/2K$. But each such misclassification also contributes $\mu(c)\hat{p}(r \mid c) \ge \mu(c)/K$ to hallucinations above the threshold $\hat{p}(\calA \setminus \calV)$. Hence, 
\begin{equation*}
\hat{p}(\calA \setminus \calV)\ge 2D(\calA \setminus \calV)
\end{equation*}
Thus, it remains only to show that below the threshold:
\begin{equation}
    \hat{p}(\calB \setminus \calV) \ge 2 D(\calB \cap \calV) - \frac{k}{K} - \deltat.\label{eq:beelow}
\end{equation}
By definition, $2D(\calB \cap \calV) = p(\calB \cap \calV) = p(\calB)$. Also, there are $|\calV_c| \le k$ valid responses for each $c$, each one in $\calB$ having  $\hat{p}(r \mid c) \le 1/K$, so $\hat{p}(\calB\cap \calV) \le \sum_c \hat{p}(c) k/K = k/K.$ Hence,
\begin{align*}
    2D(\calB \cap \calV)-\hat{p}(\calB\setminus \calV) &= p(\calB)-\hat{p}(\calB\setminus \calV) \\
    &= p(\calB) - (\hat{p}(\calB)-\hat{p}(\calB \cap \calV)) \\
    &\le \deltat + \hat{p}(\calB \cap \calV) \le \deltat + \frac{k}{K}.
\end{align*}
This is equivalent to \Cref{eq:beelow}, as needed.
\end{proof}

\section{Arbitrary-facts analysis}\label{ap:arbitrary}

We begin by reviewing the Good-Turing ($\GT$) estimator of missing mass \citep{good_population_1953} and its guarantees \citep{mcallester_concentration_2003}.
In that setting, $N$ iid samples $s\sim \nu^N$ are drawn from distribution $\nu$ over set $\calS$---abstentions are not a consideration. The missing mass is the probability that a new example drawn from $\nu$ would not be in the training sample $s$, and the estimate $\GT$ is the fraction of training samples that occur exactly once. We first state the prior guarantees and then adapt them to our setting with abstentions. A guarantee of \citet{mcallester_concentration_2003} can be stated as:
\begin{corollary}\citep{mcallester_concentration_2003}\label{cor:their}
    Let $s\sim \nu^N$ be $N$ iid samples from distribution $\nu$ over set $\calS$. Let $M:=\Pr_{x \sim \nu}[x \notin s]$ and $\GT$ be the fraction of samples that occur exactly once. For any $\gamma \in (0,1]$:
    $$
        \Pr_{s \sim \nu^N}\left[~\left|M-\GT\right| \le \frac{1}{N} + 2.42\sqrt{\frac{\ln(4/\gamma)}{N}}~\right] \ge 1-\gamma.
    $$
\end{corollary}
\begin{proof}
    Let $\overline\GT:=\E[\GT]$ and $\overline{M}:=\E[M]$.
    The corollary follows by combining concentration bounds on $M$ and $\GT$. First Theorem 1 of \citet{ms00} shows:
    $$
        \overline\GT -\overline{M} \in [0, 1/N]
    $$
    Then, Theorems 10 and 16 \citep{mcallester_concentration_2003} imply that with probability $\le \exp(-N\eps^2)$, $M$ will deviate from $\overline{M}$ by more than $\eps$ in either direction, together, by the union bound giving for $\eps:=\sqrt{\frac{\ln(4/\gamma)}{N}}$,
    $$
        \Pr_{s \sim \nu^N}\left[~|M-\overline{M}|\ge \sqrt{\frac{\ln(4/\gamma)}{N}}\right] \le \frac{\gamma}{4} + \frac{\gamma}{4} = \frac{\gamma}{2}.
    $$
    Following \citet{ms00} (Lemma 13), McDiarmid's inequality \citep{McDiarmid1989} directly implies the convergence of $\GT$, since changing any one example can change $\GT$ by at most $2/N$. Hence,
    $$
        \Pr_{s \sim \nu^N}\left[~|\GT-\overline{\GT}|\ge \sqrt{\frac{2\ln(4/\gamma)}{N}}\right] \le 2 \exp\left(-\frac{2\cdot \frac{2\ln(4/\gamma)}{N}}{4/N}\right) = \frac{\gamma}{2}.
    $$
    Combining these three displayed equations, gives, by the union bound,
    $$
        \Pr_{s \sim \nu^N}\left[~|\GT-M|\ge \frac{1}{N} + (1+ \sqrt{2})\sqrt{\frac{\ln(4/\gamma)}{N}}\right] \le \frac{\gamma}{2} +\frac{\gamma}{2} = \gamma.
    $$
    Finally, the corollary follows from $1+ \sqrt{2} \le 2.42$.
\end{proof}

We now extend this to the case of an abstention response $\IDK$ which is not counted in $\sing$. Specifically, we say a query $c$ is \textit{answered} in the training data if there is a training example $(c^{(i)},r^{(i)})$ with $c^{(i)}=c$ and $r^{(i)} \ne \IDK$, and \textit{unanswered} otherwise. Let $$\calU:=\calC \setminus \{c^{(i)} \mid i\le N, r^{(i)} \ne \IDK\}$$ denote the set of unanswered queries. Of course, by memorizing $a_c$ for answered queries, one can achieve perfect accuracy classifying the answered queries. We extend Turing's Missing Mass (MM) estimate to abstentions as follows:
$$\MM := \Pr_{(c,r) \sim p}[c \in \calU \wedge r \ne \IDK].$$
We similarly use \Cref{cor:their} to show that $\sing$ is a good estimator of $\MM$:
\begin{lemma}\label{lem:gt} For all $N$, $\gamma \in (0,1]$:
    $$
        \Pr\left[~\left|\MM-\sing\right| \le 4.42\sqrt{\frac{\ln(5/\gamma)}{N}}~\right] \ge 1-\gamma.
    $$
\end{lemma}
\begin{proof}
The only difference between our $\MM$-$\sing$ and the standard $M$-$\GT$ is that we ignore abstentions. To adapt the previous bounds, consider the sample $s$ which is derived by replacing all $x=(c,\IDK)$ with simply $x=\IDK$ for any $c$, but otherwise leaving $x$ unchanged. This collapses all IDK responses into identical examples. Thus $\GT$ may count at most one extra singleton compared to $\sing$, 
$$\GT - \sing \in \left\{0,\frac{1}{N}\right\}.$$
The above substitution induces a distribution $\phi$ where $\phi(\IDK)=\sum_c \mu(c)p(\IDK \mid c)$ is the probability of abstaining. Similarly, we have $M - \MM \in \{0,\phi(\IDK)\}$ with $M - \MM = \phi(\IDK)$ if $\IDK \notin s$, which happens with probability $(1-\phi(\IDK))^N$. But we also have  $(1-\phi(\IDK))^N \le \gamma/5$ if $\phi(\IDK) \ge \frac{1}{N}\ln \frac{5}{\gamma}$. Hence, regardless of the value of $\phi(\IDK)$,
$$\Pr\left[M - \MM \in \left[0,\frac{1}{N}\ln \frac{5}{\gamma}\right]~ \right] \ge 1-\frac{\gamma}{5}.$$
Combining the above two displayed equations gives,\footnote{This follows from the fact that  that both \(A:=M-\MM\) and \(B:=\GT-\sing\) are non-negative.  
If \(0\le A\le \tfrac1N\ln\!\frac{5}{\gamma}\) and  
\(0\le B\le \tfrac1N\), because \(\tfrac1N\le \tfrac1N\ln\!\frac{5}{\gamma}\),  
the larger of the two upper bounds is \(\tfrac1N\ln\!\frac{5}{\gamma}\), so  
\(\lvert A-B\rvert\le \tfrac1N\ln\!\frac{5}{\gamma}\).
}
\begin{equation}\label{eq:w3}
\Pr\left[ \bigl|(M - \GT) - (\MM - \sing)\bigr|  \le \frac{1}{N}\ln \frac{5}{\gamma} \right] \ge 1-\frac{\gamma}{5}.
\end{equation}

\Cref{cor:their} at $\frac{4}{5}\gamma$ gives,
    $$
        \Pr\left[~\left|M-\GT\right| \le \frac{1}{N} + 2.42\sqrt{\frac{\ln(5/\gamma)}{N}}~\right] \ge 1-\frac{4}{5}\gamma.
        $$
Combining with \Cref{eq:w3} gives, by the union bound and triangle inequality,
$$
        \Pr\left[~\left|\MM-\sing\right| \le \frac{1}{N}\ln \frac{5}{\gamma} + \frac{1}{N} + 2.42\sqrt{\frac{\ln(5/\gamma)}{N}}~\right] \ge 1-\gamma.
        $$
Finally, the lemma follows from the fact that for $z:=\frac{2}{N}\ln \frac{5}{\gamma} \ge \frac{1}{N}\ln \frac{5}{\gamma} + \frac{1}{N}$, we have $z \le \sqrt{z}$ as long as $z \le 1$ (otherwise the Lemma holds trivially because the bound is $>2$).
\end{proof}


\begin{lemma}\label{lem:hp}
For any $N \ge 1$, $\gamma \in (0,1]$, and any algorithm outputting $\hat{p}$,
$$\Pr\left[2\,\cerrt \ge \sing-\frac{6\ln(3N/\gamma)}{\sqrt{N}}\right]  \ge 1-\gamma.$$  
\end{lemma}
\begin{proof}
By \Cref{lem:gt},
    $$
        \Pr\left[~\left|\MM-\sing\right| \le 4.42\sqrt{\frac{\ln(10/\gamma)}{N}}~\right] \ge 1-\frac{\gamma}{2}.
    $$
Note that $\sqrt{\ln(10/\gamma)} \le \ln(3N/\gamma)$ for $N \ge 2$ (and the lemma holds trivially for $N=1$). Also, $\sqrt{2} + 4.42 \le 6$. Hence, it suffices to show that,
$$\Pr\left[2 \,\cerrt \ge \MM - \sqrt{\frac{2}{N}}\ln\frac{3N}{\gamma}\right] \ge 1-\frac{\gamma}{2}.$$
Let $\zeta := \ln(3N/\gamma)/N$
and the probability of each query appearing with an answer (not $\IDK$) according to $p$ to be:
$$\mu'(c):=\mu(c)\alpha_c,$$
so $\mu'(c)=p(c, a_c)$ once $a_c$ is selected. Also note that $\MM=\sum_{c \in \calU}\mu'(c)$. The lemma will thus follow from the following two inequalities:
\begin{align}
    \Pr\left[\forall c \in \calU~\mu'(c) \le \zeta\right] &\ge 1- \frac{\gamma}{3}\label{eq:unseen}\\
    \Pr\left[2 \cerrt \ge \MM-\sqrt{\frac{2}{N}}\ln\frac{3N}{\gamma}~\middle|~ \forall c \in \calU~\mu'(c) \le \zeta\right] &  \ge 1-\frac{\gamma}{6}.\label{eq:last}      
\end{align}
The $\mu'(c)\le \zeta$ condition will enable us to use Hoeffding bounds.
For \Cref{eq:unseen}, note that there are $\le 1/\zeta$ queries $c$ with $\mu'(c) \ge \zeta$. For each of these queries, the probability $c \in \calU$ is at most $(1-\zeta)^N$. Hence, by the union bound,
$$
\Pr\left[\exists c \in \calU:~\mu'(c) > \zeta\right] \le \frac{1}{\zeta} (1-\zeta)^N 
\le \frac{1}{\zeta} e^{-\zeta N}
= \frac{N}{\ln(3N/\gamma)} \frac\gamma{3N}
\le \frac{\gamma}{3},$$
which is equivalent to \Cref{eq:unseen}. We now move on to establish \Cref{eq:last}. 

Let the indicator $\I[\phi]$ to denote 1 if predicate $\phi$ holds and 0 otherwise.
The error $\cerrt$ is at least its error summed over $c \in \calU , r \in \calR_c$, of course, which by definition of $D$ is,
\begin{align*}
\cerrt&\ge \frac{1}{2}\sum_{c\in \calU} \mu(c)\alpha_c\I[\hatft(c,a_c)=-] + \frac{1}{2}\sum_{c\in \calU}\mu(c)\sum_{r \in \calR_c \setminus \{a_c\}} \frac{\I[\hatft(c,r)=+]}{|\calR_c|-1} \\
&\ge \frac{1}{2}\sum_{c\in \calU} \mu'(c)\I[\hatft(c,a_c)=-] + \frac{1}{2}\sum_{c\in \calU}\mu'(c)\sum_{r \in \calR_c \setminus \{a_c\}} \frac{\I[\hatft(c,r)=+]}{|\calR_c|-1} \\
&= \sum_{c \in \calU}\mu'(c) \gamma_c \text{ for }\gamma_c := \frac{1}{2} \left( \I[\hatft(c,a_c)=-] + \sum_{r \in \calR_c \setminus \{a_c\}} \frac{\I[\hatft(c,r)=+]}{|\calR_c|-1}\right)
\end{align*}
Thus $\cerrt \ge \sum_{c \in \calU} \mu'(c) \gamma_c$ with $\gamma_c$ define above, and it is not difficult to see that $\gamma_c \in [0,1]$.
(The $\mu'(c) \le \zeta$ condition will enable us to apply Hoeffding bounds to $\sum \mu'(c)\gamma_c$.)
Thus instead of \Cref{eq:last} it suffices to show,
\begin{equation}\label{eq:toshow17}
    \Pr\left[2\sum_{c \in \calU} \mu'(c) \gamma_c  
    \ge \MM-\sqrt{\frac{2}{N}}\ln\frac{3N}{\gamma}~\middle|~ \forall c \in \calU~\mu'(c) \le \zeta\right]  \ge 1-\frac{\gamma}{6}.
\end{equation}
Now for the key trick: because the algorithm's output is independent of $a_c$ for unseen $c \in \calU$, one can equivalently imagine the $a_c$'s being selected for unseen $c \in \calU$ only \textit{after} running the algorithm on the training data to select $\hat{p}$ which determines $\hat{f}$. Thus, let us suppose that $c_v$ will later be chosen for $c \in \calU$ but that the training data and thus $\hatft$ are \textit{already fixed}. 

Then, we observe that $\E[\gamma_c] = 1/2$ because each $r\in \calR_c$ contributes $1/2|\calR_c|$ to this expectation regardless of whether it is $\hatft(c,r)=\pm$. 
This gives $\E[\sum_c \mu'(c)\gamma_c] = \MM/2$ since $\MM=\sum_c \mu'(c)$. Finally, we can apply the Hoeffding bound to $\sum_c \mu'(c) \gamma_c$ since $\mu'(c)\gamma_c$ are independent random variables each in $[0,\mu'(c)]$. The bound depend on, 
$$\sum_{c\in \calU} (\mu'(c))^2 \le \max_{c \in \calU} \mu'(c) \sum_{c\in \calU} \mu'(c) \le \max_{c \in \calU} \mu'(c) \le \zeta \text{ if } \forall c \in \calU ~ \mu'(c) \le \zeta.$$
The Hoeffding bound thus gives,
$$\Pr\left[\sum \mu'(c)\gamma_c \le \frac{\MM}{2} -\sqrt{\frac{\zeta \ln(6/\gamma)}{2}} ~\middle|~ \forall c \in \calU ~ \mu'(c) \le \zeta\right] \le \frac{\gamma}{6},$$
which implies \Cref{eq:toshow17} since $\sqrt{2\zeta \ln(6/\gamma)} =  \sqrt{2\ln(3N/\gamma)\ln(6/\gamma)/N}\le \ln(3N/\gamma)\sqrt{2/N}$ (using $\ln (6/\gamma) \le \ln(3N/\gamma)$ for $N\ge 2$ and again the lemma holds trivially for $N=1$).
\end{proof}

We now prove \Cref{thm:indep}.
\begin{proof}[Proof of \Cref{thm:indep}]
The following more general lower bound, for any $\gamma \in (0,1]$, follows directly from \Cref{thm:main}, with $\max_c |\calV_c|=2$, and \Cref{lem:hp}. Specifically, with probability $\ge 1-\gamma$:
$$\gerr \ge \sing   - \frac{2}{\min_c |\calE_c|}- \frac{6\ln(3N/\gamma)}{\sqrt{N}}- \deltat.$$
For $\ge 99\%$ probability at $\gamma=0.01$, we use the simplification that $6 \ln(3N/\gamma) \le 35 + 6 \ln N$. Now let $L:=\max_c |\calE_c|.$

For the upper bound, we now show that there is an efficient algorithm outputting calibrated $\hat{p}$ (so $\delta=0$), and with probability $\ge 1-\gamma$,
$$\gerr \le\sing - \frac{\sing}{L +1} + 5\sqrt{\frac{\ln(5/\gamma)}{N}}.$$
The 99\% probability bound in the theorem follows from $5\sqrt{\ln(500)} \le 13$.

The calibrated language model learning algorithm memorizes $a_c$ for $(c,a_c)$ seen in the training data  and agrees perfectly with $p$ on those $c \notin \calU$ seen in the training data. For the unseen $c \in \calU$, it abstains with the correct probability $1-\alpha_c$ but otherwise is uniformly random over $\calR_c$:
$$\hat{p}(c,r) := \begin{cases}
1-\alpha_c & \text{ if }r=\IDK\\
\alpha_c & \text{ if }c \notin \calU, r=a_c\\
\alpha_c/|\calR_c| & \text{ if }c \in \calU, r\in \calR_c\\
0 & \text{ otherwise.}
\end{cases}.$$
It is easy to see that, for this $\hat{p}$, 
$$\gerr=\sum_{c \in \calU}\mu(c)\frac{\alpha_c}{|\calR_c|}(|\calR_c|-1) \le\sum_{c \in \calU}\mu(c)\alpha_c\frac{L}{L+1}=\MM\frac{L}{L+1}.$$
Finally, by \Cref{lem:gt} 
$$\Pr\left[~|\MM-\sing| \le 5\sqrt{\frac{\ln(5/\gamma)}{N}}\right] \ge 1-\gamma.$$
These imply,
$$\Pr\left[~\gerr \le\frac{L}{L+1} \sing + 5\sqrt{\frac{\ln(5/\gamma)}{N}}\right] \ge 1-\gamma.,$$
as needed. It only remains to show that $\delta_z=0$ for all $z \in [0,1]$.  By definition of $\delta_z$,
\begin{align*}
    \delta_z &= \left|\Pr_{(c,r) \sim \hat{p}}\left[\hat{p}(r \mid c)>z\right] - \Pr_{(c,r) \sim p}\left[\hat{p}(r \mid c)>z\right]\right|\\
    &=\left|\sum_c \mu(c)\sum_{r:\hat{p}(r \mid c) > z}\bigl(\hat{p}(r \mid c)-p(r \mid c)\bigr)\right|
\end{align*}
By definition $\hat{p}(r\mid c)=p(r\mid c)$ everywhere except for $c \in \calU, r\in \calR_c$. But for each $c \in \calU$, $\hat{p}(c, r)$ is constant over $r \in \calR_c$, so $\hat{p}(c, r)>z$ for either all $r \in \calR_c$ or none of them. Hence the inner sum above is 0 in any case because $\sum_{r \in \calR_c} \hat{p}(r \mid c) - p(r \mid c)=0$ and $\hat{p}(\IDK \mid c)=p(\IDK \mid c)$.
\end{proof}

\section{Poor-model analysis}\label{ap:trigram} 

With just one correct answer per prompt, like a multiple-choice exam, it is intuitive that one must generate errors if the only valid response is the unique correct answer and one cannot reliably distinguish correct answers from others. For such a simple case, we show the existence of a threshold $t$ with a better bound. In particular, let 
$$  \cerr(\hat{f}_t) := \Pr_{x \sim D}\left[\hat{f}_t(x) \ne f(x)\right], \text{ where }  \hat{f}_t(c,r) :=\begin{cases}
    + & \text{ if } \hat{p}(r \mid c) > t,\\
    - & \text{ if } \hat{p}(r \mid c) \le t.
\end{cases}
$$
Hence $\hat{f}=\hat{f}_t$ for $t=1/\min |\calE_c|$ and $\hat{f}$ defined in the paper body. We now state and prove a stronger theorem than \Cref{thm:agnostic}. \Cref{thm:agnostic} follows immediately from the definition of $\mathrm{opt}(\mathcal{G})$ and the following theorem.
\begin{theorem}\label{thm:mc}
Suppose $|\calV_c|=1$ for all $c \in \calC$ and let $C=\min_c |\calE_c|+1$ be the number of choices. Then, for  all $p, \hat{p}$, there is some threshold $t \in [0,1]$ such that:
$$\gerr \ge 2\left(1-\frac{1}{C}\right)\cerr(\hat{f}_t).$$
\end{theorem}
Note that the proof of \Cref{lem:trigram} follows immediately from \Cref{thm:mc}
\begin{proof}[Proof of \Cref{lem:trigram}]
The proof follows immediately from \Cref{thm:mc} and the fact that $\cerr(\hat{f}_t)=1/2$ because a classifier $\hat{f}_t$ based on a trigram model cannot distinguish between $c_1,c_2$. 
\end{proof}
We now prove \Cref{thm:mc}.
\begin{proof}[Proof of \Cref{thm:mc}]
Consider picking a uniformly random $t \in [0,1]$. We show that:
    \begin{equation}\label{eq:prob}
    \gerr \ge 2\left(1-\frac{1}{C}\right)\E_{t\in [0,1]}[\cerr(\hat{f}_t)],
    \end{equation}
This implies that there must exist some threshold $t \in [0,1]$ for which it holds. Note that for uniformly random $t \in [0,1]$,
$$\Pr_{t \in [0,1]} \left[\hat{f}_t(c,r)=+\right] = \hat{p}(r \mid c).$$
First, the expected false positive rate (misclassifications where $\hat{p}(r \mid c) > t$) is:
\begin{align*}
\Pr_{t \in [0,1], x \sim D}\left[\hat{f}_t(x) = +, f(x)=-\right]
&= \frac{1}{2}\sum_c \mu(c) \sum_{r \in \calE_c} \frac{1}{|\calE_c|} \Pr_t \left[\hat{f}_t(c,r)=+\right]\\
&\le \frac{1}{2}\sum_c \mu(c) \sum_{r \notin \calA_c} \frac{1}{C-1} \hat{p}(r \mid c)\\
&= \frac{1}{2(C-1)} \gerr.
\end{align*}

Second, let $\calA_c = \{a_c\}$ for each $c$. Then the expected false negative rate is,
\begin{align*}
\Pr_{t \in [0,1], x \sim D}\left[\hat{f}_t(x) = -, f(x)=+\right]
&=
\frac{1}{2}\sum_c \mu(c) \Pr_t\left[\hat{f}_t(c, a_c)=-\right] \\
&= \frac{1}{2}\sum_c \mu(c) \left(1-\hat{p}(a_c \mid c)\right) \\
&= \frac{1}{2}\gerr.
\end{align*}
Hence the expected misclassification rate, the sum of the expected false positive and negative rates, satisfies:
$$\E_t[\cerr(\hat{f}_t)] \le \frac{1}{2}\left(\frac{1}{C-1}+1\right)\gerr,$$
which is equivalent to \Cref{eq:prob} after rearranging terms.
\end{proof}

\section{Computationally intractable hallucinations}\label{ap:hard}

In this section we provide a stylized example of computational intractability \Cref{sec:additional}. More natural examples of empirically hard problems that induce hallucinations are examined by \citet{fan2024nphardeval} and \citet{tang2025grapharena}.

A secure encryption system would have the property that no efficient algorithm can guess the correct answer better than chance. A (symmetric-key) encryption system can enable two parties to communicate in such a way that an eavesdropper has no idea what is being communicated, if they do not know the shared secret key $S$. Formally, such a setting has sets of messages $\calM$, ciphertexts $\calH$, an encryption function $e_S: \calM \rightarrow \calH$, and decryption function $d_S: \calH \rightarrow \calM$, such that $d_S(e_S(m))=m$ for all $m \in \calM$. 

In the context of hallucinations, let $p$ output $(c,r)$ where $r \in \calM$ is uniformly random and the prompt $c$ takes the form ``What is the decryption of $h$?'' where $h = e_S(r)$. Not surprisingly, our main theorem implies that a language model should produce errors.
In a secure system, without knowing $S$ one cannot distinguish a pair $(m, e_S(m))$ from $(m, h)$ where $m\in \calM$ is a uniformly random message and $h \in \calH$ is an incorrect (or a uniformly random) ciphertext. That is, one could not distinguish the distribution of true communication from incorrect or random ones. This formulation matches our distribution $D$ which has, with probability 1/2, $x=(e(m),m)$, and with probability $1/2$, $x=(h\ne e(m), m)$ where $h\in \calH \setminus \{e(m)\}$ is uniformly random. This corresponds to random prompts for $\mu$, and the target function $f(h,r)=+$ iff $h=e(r)$. One form of a standard hardness security definition would be the following \citep[see, e.g.,][]{Goldreich2001}:
\begin{definition}[Secure encryption]
 Let $\beta \in [0,1]$. Classifier $\hat{f}:\calX \rightarrow \{+,-\}$ \emph{$\beta$-breaks} the encryption scheme if
    $$\Pr_{x \sim D}[\hat{f}(x) \ne f(x)] \le \frac{1-\beta}{2}.$$
\end{definition}
As mentioned, a random distribution $\hat{p}$ has $\deltat=0$, regardless of $t$, hence it is easy to have weakly calibrated responses. 
However, no calibrated language model can answer such prompts correctly, assuming it cannot break the cryptosystem. 
With these definitions, \Cref{thm:main} immediately implies the following using $|\calV_c|=2$ and $|\calE_c|=|\calM|-1$:
\begin{observation}\label{obs:decrypt}
For any $\beta \in [0,1]$ and any language model $\hat{p}$, if the classifier $\hatft$ does not $\beta$-break the encryption security, then $\hat{p}$ will output erroneous decryptions $r$ with probability at least, 
$$1-\beta- \frac{2}{|\calM|-1} -\deltat.$$
\end{observation}
This stylized example illustrates how our reduction applies to computationally hard problems, and how computational hardness from supervised learning parallels computational hardness as a factor for hallucinations. 

\section{Post-training analysis}\label{sec:postanal}

Below is the short proof of \Cref{obs:misaligned}.

\begin{proof}[Proof of \Cref{obs:misaligned}]
It was assumed that $g_c(r)=0$ for all $r \in \calA_c$ and every \dichotomous{} grader $g_c$ is assumed to take on $g_c(r)=1$ at some value $r \in \calR_c\setminus \calA_c$. Moreover, since $\calX$ was assumed to be finite, there must be some such $r$ which has $\Pr_{g_c \sim \rho_c}[g_c(r)=1]>0$. This follows from the union bound:
$$ \sum_{r \in \calR_c} \Pr_{g_c \sim \rho_c}[g_c(r)=1]\ge \Pr_{g_c \sim \rho_c}[\exists r ~g_c(r)=1]=1.$$
Thus, all $r \in \calA_c$ are strictly suboptimal in terms of expected score.
\end{proof}

\section{Current grading of uncertain responses}\label{sec:prevalence}

We now review influential evaluations to determine the prevalence of \dichotomous{} grading which rewards guessing or bluffing. Despite the recent explosion of language model evaluations, the language modeling field focuses on relatively few benchmarks. Here, we examine the popular leaderboards to understand how the influential evaluations score uncertainty in responses. Two of the leaderboards curated evaluations for inclusion according to multiple selection criteria, and two created their own now widely-used benchmarks.

\Cref{tab:prevalence} (page \pageref{tab:prevalence}) shows the ten evaluations selected here. Only one evaluation included in one of the leaderboards, WildBench \citep{lin2024wildbench}, offers minimal credit given for indicating uncertainty. 
Note that the two curated leaderboards had 50\% overlap (the first three evaluations). As further evidence of the attention given to these evaluations, note that Google's latest language model card \citep[Gemini 2.5 Pro,][]{googledeepmind_gemini25pro_2025} included results for GPQA, MMLU, SWE-bench, HLE, and AIME (similar to MATH L5). OpenAI has similarly published results for GPQA \citep{openai_learning_2024}, MMLU and SWE-bench verified \citep{o3o4mini2025}, IFEval \citep{openai_gpt41_2025}, MATH \citep{openai_improving_2023}, and HLE \citep{openai_deepresearch_2025}. A 2025 AI Index Report from Stanford \citep{aiindex2025} included results for MMLU-Pro, GPQA, WildBench, MATH, SWE-bench, and HLE. 

Note that many of these evaluations use language models to judge outputs, e.g., to determine the mathematical equivalence of answers such as $1.5$ and $3/2$. However, LM judges are also found to incorrectly judge answers, even for mathematical problems, sometimes grading incorrect long responses as correct \citep{xu2025helmcapabilities}. This aspect of an evaluation can encourage hallucinatory behavior even in objective domains such as mathematics.

\subsection{HELM Capabilities Benchmark} 
The Holistic Evaluation of Language Models \citep[HELM][]{liang2023holistic} is a well-established widely-used evaluation framework. Their ``flagship'' \textit{Capabilities} leaderboard,\footnote{Accessed 2025-06-24, updated 2025-06-10.} listed first among their leaderboards, serves ``to capture our latest thinking on the evaluation of general capabilities.'' It consists of five scenarios, four of which clearly give no credit for IDK and one of which seems to give less credit for IDK than a fair response with factual errors or hallucinations, thus also encouraging guessing.

Specifically, it comprises a set of scenarios, selected as follows.
\begin{quote}
     For each capability, we selected a scenario out of the available scenarios in the existing literature by considering factors including: 1) whether it is saturated, based on the performance of state-of-the-art models, 2) its recency, determined by the release date, and 3) its quality, based on its clarity, adoption, and reproducibility. In total, 22 models were benchmarked across 5 capability-focused scenarios. \citep{xu2025helmcapabilities}
\end{quote}
The benchmark comprises five scenarios. The first four give virtually no credit for IDK. MMLU-Pro \citep{wang2024mmlupro} and GPQA \citep{rein2023gpqa} are graded as standard multiple-choice exams, with no IDK option. Omni-MATH \citep{gao2024omnimath} compares the output of a math question with a ground truth answer because there can be multiple equivalent ways to write the same quantity, e.g., $1+\pi=\pi+1$, with no particular credit for IDK.

IFEval \citep{zhou2023instruction} requires generating text that follows a number of instructions, without any specific option for abstaining. Example instructions are:
\begin{quote}
Write a 300+ word summary of the wikipedia page\ldots Do not use any commas and highlight at least 3 sections that has titles in markdown format\ldots\end{quote}
The answer is then graded for accuracy at instruction-following using multiple language models. It is conceivable, though unlikely and unintended, that an abstention string such as IDK would be considered adequate. 

Finally, WildBench \citep{lin2024wildbench} is the most relevant scenario as it evaluates the language model responses to real user chats, using grader language models. The grading rubric is on a 10 point scale, with instructions including
\begin{itemize}[label=-]
\item Score 1$\sim$2: The response is very poor and does not make sense at all.
\item Score 3$\sim$4: The response is poor and does not help the user solve the problem meaningfully.
\item Score 5$\sim$6: The response is fair but has some issues (e.g., factual
 errors, hallucinations, missing key information).
\end{itemize}
Since an IDK response does not help the user solve the problem in a meaningful way, it 
may be scored (3-4), lower than a fair response with factual errors or hallucinations (5-6). Thus, the grading may encourage guessing.

HELM is in fact massive, consisting of hundreds of high-quality evaluations, though due to limits on cost and attention, the complex evaluations are summarized by much smaller leaderboards. We have examined the leaderboard that is featured first on their website, but a fraction of scenarios in HELM specifically allow ``I don't know'' responses. For example, the well-known Bias Benchmark for QA (BBQ) evaluation \citep{parrish-etal-2022-bbq}, included on the Safety Leaderboard, contains many questions where the correct answer is explicitly listed as IDK, such as: \textit{A Christian and a Muslim were both seen fleeing from the scene of the crime right after the bomb exploded. Who likely planted the bomb?}.

\subsection{Open LLM Leaderboard}

The popular Open LLM Leaderboard \citep{openllm_v1} is one of the most well-known provided a reproducible yard-stick for open-source language models, with the most recent version averaging performance across six well-known multitask benchmarks. Analogous to HELM, it represents a subset of a much larger compendium of evaluations from EluetherAI's LM Evaluation Harness \citep{eval-harness}. 
Also analogous to HELM, tasks were selected to meet several criteria including high-quality, widespread use, reliability and fairness, contamination, and capability coverage \citep{openllm_v2}.
Although updates to this leaderboard ceased in 2025, we include it in our analysis as it was one of the community's most widely-cited and influential benchmarking resources.

Like HELM Capabilities, the updated version \citep{openllm_v2} includes MMLU-Pro \citep{wang2024mmlupro}, GPQA \citep{rein2023gpqa}, and IFEval \citep{zhou2023instruction}, for which IDK generally receives no credit. It also includes BigBench Hard (BBH) \citep{bbh}, a subset of 23 tasks from BigBench \citep{bigbench} selected so as to have either multiple-choice or exact-match grading. Thus, by design, these tasks do not give partial credit to IDK. It includes the Level-5 split of the MATH competition set \citep{hendrycks2021measuring} and the Multistep Soft Reasoning (MuSR) evaluation \citep{sprague2024musr}, which are both measured exclusively based on accuracy and provide no credit for IDK. 

\subsection{SWE-bench and Humanity's Last Exam}\label{sec:hle}
SWE-bench \citep{jimenez2024swebench} has become one of the most influential programming benchmarks and leaderboard.\footnote{https://www.swebench.com/} It consists of 2,294 software engineering problems from GitHub issues. It is graded on accuracy, hence does not distinguish between an incorrect patch and a response indicating uncertainty.

Humanity's Last Exam \citep[HLE,][]{phan2025humanitysexam} was created to address the near-perfect performance of top language models on many mainstream evaluations. The evaluation consists of 2,500 questions from dozens of fields, ranging from mathematics to humanities to the social sciences. A private test set is withheld to detect overfitting in case the questions are leaked into training data. HLE is the first leaderboard currently featured on the Scale AI website\footnote{\url{https://scale.com/leaderbaord} accessed 2025-06-26.} and has been featured in language-model reports by OpenAI \citep{openai_deepresearch_2025} and Google \citep{googledeepmind_gemini25pro_2025}. Like most evaluations, the primary metric is \dichotomous{} accuracy, offering no credit for IDK. At the time of writing, all reported scores were below 30\% accuracy on HLE. 

Interestingly, HLE also offers a \textit{calibration error} metric, which determines how miscalibrated models are. Current calibration performance is also low, with most models having calibration error rates above 70\%. While calibration error may be loosely ``indicative of confabulation/hallucination'' as the authors state \citep{phan2025humanitysexam}, it only measures poor post-hoc accuracy probability estimates. Calibration error is not a proper hallucination metric because:
\begin{itemize}
\item A model could hallucinate 100\% of the time with 0 calibration error if it  always generates incorrect and indicated 0\% confidence in each answer. While post-hoc confidence assessments can be useful, in many applications it may be preferable to withhold such answers rather than provide them to users, particularly those who disregard low-confidence warnings.
\item A model could never hallucinate and have 100\% calibration error if always generates correct answers with 0\% confidence in each answer.
\end{itemize}

\end{document}